\documentclass[10pt,twocolumn,letterpaper]{article}

\usepackage{cvpr}
\usepackage{times}
\usepackage{epsfig}
\usepackage{graphicx}
\usepackage{amsmath}
\usepackage{amssymb}
\usepackage{booktabs}
\usepackage{mathtools}
\usepackage{multirow}
 \usepackage{amsthm}
\newtheorem{lemma}{Lemma}

\usepackage{etoolbox}
\newcommand{\zerodisplayskips}{%
  \setlength{\abovedisplayskip}{2pt}%
  \setlength{\belowdisplayskip}{2pt}%
  \setlength{\abovedisplayshortskip}{2pt}%
  \setlength{\belowdisplayshortskip}{2pt}}
\appto{\normalsize}{\zerodisplayskips}
\appto{\small}{\zerodisplayskips}
\appto{\footnotesize}{\zerodisplayskips}

\graphicspath{ {figures/} }

\usepackage{hyperref}
\hypersetup{pagebackref=true,breaklinks=true,colorlinks}

\cvprfinalcopy 

\def\cvprPaperID{6069} 

\ifcvprfinal\fi
\begin{document}

\title{ScopeFlow: Dynamic Scene Scoping for Optical Flow}

\author{Aviram Bar-Haim$^1$ and Lior Wolf$^{1,2}$ \\
$^{1}$Tel Aviv University\\
$^{2}$Facebook AI Research\\
}

\maketitle
\thispagestyle{empty}

\begin{abstract}
We propose to modify the common training protocols of optical flow, leading to sizable accuracy improvements without adding to the computational complexity of the training process. The improvement is based on observing the bias in sampling challenging data that exists in the current training protocol, and improving the sampling process. In addition, we find that both regularization and augmentation should decrease during the training protocol.

Using an existing low parameters architecture, the method is ranked first on the MPI Sintel benchmark among all other methods, improving the best two frames method accuracy by more than 10\%. The method also surpasses all similar architecture variants by more than 12\% and 19.7\% on the KITTI benchmarks, achieving the lowest Average End-Point Error on KITTI2012 among two-frame methods, without using extra datasets.
\end{abstract}

\section{Introduction}

The field of optical flow estimation has benefited from the availability of acceptable benchmarks. In the last few years, with the adoption of new CNN~\cite{LeCun:1989:BAH:1351079.1351090} architectures, a greater emphasis has been placed on the training protocol.

A conventional training protocol now consists of two stages: (i) pretraining on larger and simpler data and (ii) finetuning on more complex datasets. In both stages, a training step includes the following: (i) sampling batch frames and flow maps, (ii) applying photometric augmentations to the frames, (iii) applying affine (global and relative) transformations to the frames and flow maps, (iv) cropping a fixed size random crop from both input and flow maps, (v) feeding the cropped frames into a CNN architecture, and (vi) backpropagating the loss of the flow estimation. 

While photometric augmentations include variations of the input image values, affine transformations are used to augment the variety of input flow fields. Due to the limited motion patterns represented by today's optical flow datasets, these regularization techniques are required for the data driven training. We chose the word \textit{scoping}, to define the process of affine transformation followed by cropping, as this process sets the scope of the input frames.

\begin{figure}[t]
\begin{center}
    \includegraphics[width=0.99\linewidth]{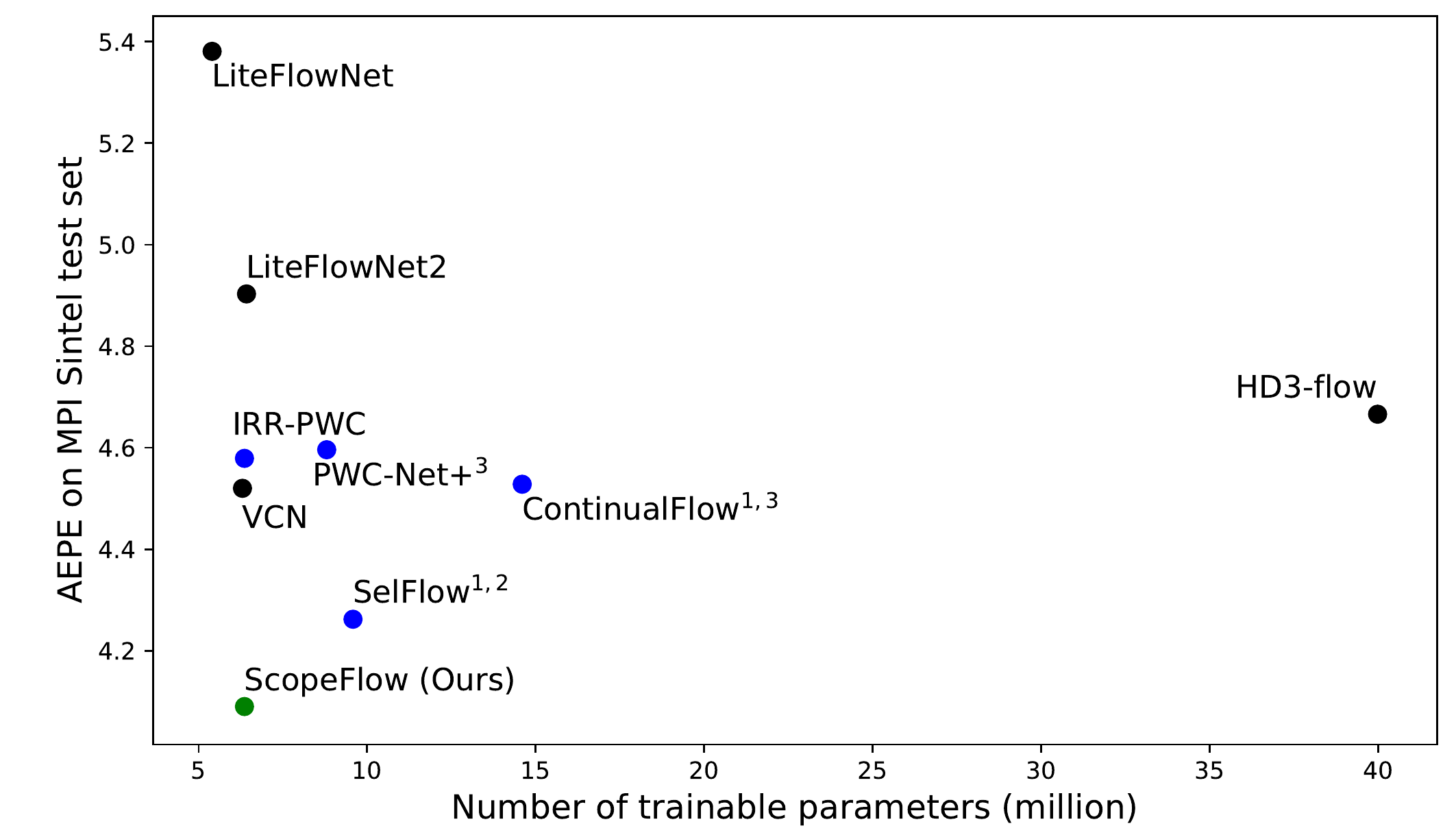}
    \caption{\textbf{Model size and accuracy trade-off.} Average-end-point-error of the leading methods on the MPI Sintel benchmark vs. the number of trainable parameters. PWC-Net based models are marked in blue. Our model is in the bottom left corner, achieving the best performance with low number of parameters during training.
    $^1$Methods that use more than two frames. $^2$SelFlow uses half of the parameters in test time. $^3$Trained with additional datasets.}
\end{center}
\label{fig:parameters_to_epe}
\end{figure}

To improve optical flow training, we ask the following questions: {Q1.} How do fixed size crops affect this training? {Q2.} What defines a good set of scopes for optical flow? {Q3.} Should regularization be relaxed after pretraining?

Our experiments employ the smallest PWC-Net~\cite{Sun2018PWC-Net} variant of Hur \& Roth~\cite{Hur2019CVPR}, with only 6.3M trainable parameters, in order to support low memory, real time processing. We demonstrate that by answering these questions and contributing to the training procedure, {\color{black}it is possible to train} a dual frame, monocular and small sized model to outperform all other models on the MPI Sintel benchmark. The trained model improves the accuracy of the baseline model, which uses the same architecture, by 12\%. See Fig.~\ref{fig:parameters_to_epe} for a comparison to other networks.

Moreover, despite using the smallest PWC-Net variant, our model outperformed all other PWC-Net variants on both KITTI 2012 and KITTI 2015 benchmarks, improving the baseline model results by 12.2\% and 19.7\%
on the public test set, and demonstrating once more the power of using the improved training protocol. 

Lastly, albeit no public benchmark is available for occlusion estimation, we compared our occlusion results to other published results on the MPI Sintel dataset, demonstrating more than 5\% improvement of the best published F1 score. 

Our main contributions are:
(i) showing, for the first time, as far as we can ascertain, that CNN training for optical flow and occlusion estimation can benefit from cropping randomly sized scene scopes, (ii) exposing the powerful effect of regularization and data augmentation on CNN training for optical flow and (iii) presenting an updated generally applicable training scheme and testing it across benchmarks, on the widely used PWC-Net network architecture.

Our code is \href{https://github.com/avirambh/ScopeFlow}{\textit{available online}} and our models are openly shared, in order to encourage follow-up work, to support reproducibility, and to provide an improved performance to off the shelf real-time models.

\section{Related work}
\label{related_works}
The deep learning revolution in optical flow started with deep descriptors~\cite{6751282,Gadot_2016_CVPR,bailer_patch_match} and densification methods~\cite{Zweig_2017_CVPR}. Dosovitskiy \etal~\cite{7410673} presented FlowNet, the first deep end-to-end network for optical flow dense matching, later improved by Ilg~\etal~\cite{IMSKDB17}, incorporating classic approaches, like residual image warping. Ranjan \& Black~\cite{spynet2017} showed that deep model size can be much smaller with a coarse to fine pyramidal structure. 
Hui \etal~\cite{hui18liteflownet,hui19liteflownet2} suggested a lighter version for FlowNet, adding features matching, pyramidal processing and features driven local convolution. Xu \etal~\cite{xu_2017_dcflow} adapted semi-global matching~\cite{4359315} to directly process a reshaped 4D cost volume of features learned by CNN, inspired by common practices in stereo matching. Yang \& Ramanan~\cite{yang2019vcn} suggested a method for directly learning to process the 4D cost volume, with a separable 4D convolution. Sun~\etal~\cite{Sun2018PWC-Net} proposed PWC-Net, which includes pyramidal processing of warped features, and a direct processing of a partial layer-wise cost volume, demonstrated strong performance on optical flow benchmarks. Many extensions were suggested to the PWC-Net architecture, among them multi-frame processing, occlusion estimation, iterative warping and weight sharing~\cite{ren2018fusion,Neoral2018ACCV,Liu:2019:SelFlow,Hur2019CVPR}.

\noindent \textbf{Pretraining optical flow models\quad}
Today's leading optical flow learning protocols, include  pretraining on large scale data. The common practice is to pretrain on the FlyingChairs~\cite{7410673} and then on FlyingThings3D~\cite{things3d} (FChairs and FThings). As shown by recent works~\cite{Mayer_2018,IMSKDB17}, the multistage pretraining ordering is critical. The FChairs dataset includes 21,818 pairs of frames, generated by CAD models~\cite{seeing_3d_chairs}, with flicker images background. FThings is a natural extension of the FChairs dataset, having 22,872  larger 3D scenes with more complex motion patterns. Hur \& Roth~\cite{Hur2019CVPR} created a version of FChairs with ground truth occlusions, called FlyingChairsOcc (denoted FChairsOcc), to allow supervised pretraining on occlusion labels.

\noindent \textbf{Datasets and benchmarks\quad} 
The establishment of larger complex benchmarks, such as MPI Sintel~\cite{Butler:ECCV:2012} and KITTI~\cite{Geiger2012CVPR,Menze2015CVPR}, boosted the evolution of optical flow models. The MPI Sintel dataset was created from the Sintel movie, composed of 25, relatively long, annotated scenes, with 1064 training frames in total. The final pass category of Sintel is a challenging one, having many realistic effects to mimic natural scenes. The KITTI2012 dataset comprises 194 training pairs with annotated flow maps, while KITTI2015 has 200 dynamic color training pairs. Furthermore, some methods are using more datasets during the finetune process, such as HD1K~\cite{kondermann2016hci}, Driving and Monkaa~\cite{things3d}.

\noindent \textbf{Motion categories\quad}
MPI Sintel provides a stratified view of the error magnitude of challenging motion patterns. The ratio of the best mean error for the small motion category (slower than 10 pixels per frame) to the large motion category (faster than 40 pixels per frame) is approximately x44. In Sec.~\ref{method:scene_scoping}, we present one possible theoretical explanation for the poor performance of state of the art methods in large motions, and suggest an approach to improve the accuracy of this pixels category.

Another example is the category of unmatched pixels. This category includes pixels belonging to regions that are visible only in one of two adjacent frames (occluded pixels). As expected, these pixels share much higher end-point-error than match-able pixels: the ratio of the best match-able EPE to the best non match-able is approximately 9.5.

Different deep learning approaches were suggested to tackle the problems of fast objects and occlusion estimation. Among the different solutions suggested were: occlusion based loss~\cite{occlusion_sep} and model~\cite{Liu:2019:DDFlow,Liu:2019:SelFlow} separation, and multi-frame support for long-range, potentially occluded, spatio-temporal matches~\cite{ren2018fusion,Neoral2018ACCV}. We suggest a new approach for applying multiple strategies online. Our findings imply that the training can be improved by applying scene scope variations, while taking into account the probability of sampling valid examples from different flow categories.

\noindent \textbf{Training procedure and data augmentation\quad}
Fleet \& Weiss~\cite{Fleet2006OpticalFE} showed the importance of photometric variations, boundary detection and scale invariance to the success of optical flow methods. 
In recent years, the importance of the training choices attracted more attention~\cite{Liu:2019:SelFlow}.  Sun~\etal~\cite{Sun2018:Model:Training:Flow} used training updates to improve the accuracy of the initial PWC-Net model by more than 10\%, showing they could improve the reported accuracy of FlowNetC (a sub network of FlowNet) by more than 50\%, surpassing FlowNet2~\cite{hui19liteflownet2} performance, with their updated training protocol. Mayer \etal~\cite{Mayer_2018} suggests that no single best general-purpose training protocol exists for optical flow, and different datasets require different care. These conclusions are in line with our findings on the importance of proper training.

\subsection{PWC-Net architectures}
\label{related_works:pwc}
PWC-Net~\cite{Sun2018PWC-Net} is the most popular architecture for optical flow estimation to date, and many variants for this architecture were suggested~\cite{ren2018fusion,Melekhov+Tiulpin+Sattler+Pollefeys+Rahtu+Kannala:2018,Hur2019CVPR,Liu:2019:SelFlow,Neoral2018ACCV}. PWC-Net architecture was built over traditional design patterns for estimating optical flow, given two temporally consecutive frames, such as: pyramidal coarse-to-fine estimation, cost volume processing, iterative layerwise feature warping and others. 
\noindent \textbf{Features warping\quad} In PWC-Net, a CNN encoder creates feature maps for the different network layers (scales). The features of the second image are backward warped, using the upsampled flow of the previous layer processing, for every layer $l$, except the last layer $l_{Top}$, by:
\begin{equation}
    \label{warping}
    c_{w}^{l}(x) = c_{2}^{l}(x + up_{\times2}(f^{l+1}(x))
\end{equation}
where x is the pixel location, $c_{w}^{l}(x)$ is the backward warped feature map of the second image, $f^{l+1}(x)$ is the output flow of the coarser layer, and $up_{\times2}$ is the $\times2$ up-sampling module, followed by a bi-linear interpolation.

\noindent \textbf{Cost volume decoding\quad} A correlation operation applied on the first and backward warped second image features, in order to construct a cost volume:
\begin{equation}
\label{eq:correlation}
    cost^{l}(x_{1}, x_{2}) = \frac{1}{N}(c_{1}^{l}(x_{1}))^{T}c_{w}^{l}(x_{2})
\end{equation}
{where $c_{n}^{l}(x)\in \mathbb{R}^N$ is a feature vector of image n.}

The cost volume is then processed by a CNN decoder, in order to estimate the optical flow directly. In some variants of PWC-Net~\cite{Neoral2018ACCV,Hur2019CVPR} there
is an extra decoder with similar architecture for occlusion estimation.

\noindent \textbf{Iterative residual processing\quad} 
Our experiments employ the Iterative Residual Refinement proposed by Hur \& Roth~\cite{Hur2019CVPR}. The reasons we chose to test our changes for the PWC-Net architecture on the IRR variant are: 
(i) IRR has the lowest number of trainable parameters among all PWC-Net variants, making a state of the art result obtained with proper training more significant, (ii) it uses shared weights that could be challenged with scope and scale changes, and if successful, it would demonstrate the power of a rich, scope invariant feature representations, (iii) this variant is using only two frames - demonstrating the power of dynamic scoping without long temporal connections, and (iv) the occlusion map allows the direct evaluation of our training procedure on occlusion estimation. Therefore, any success with this variant directly translates to real-time relatively low complexity optical flow estimation.

\section{Scene scoping analysis}

\begin{figure}[t]
    \centering
    \includegraphics[width=0.99\linewidth]{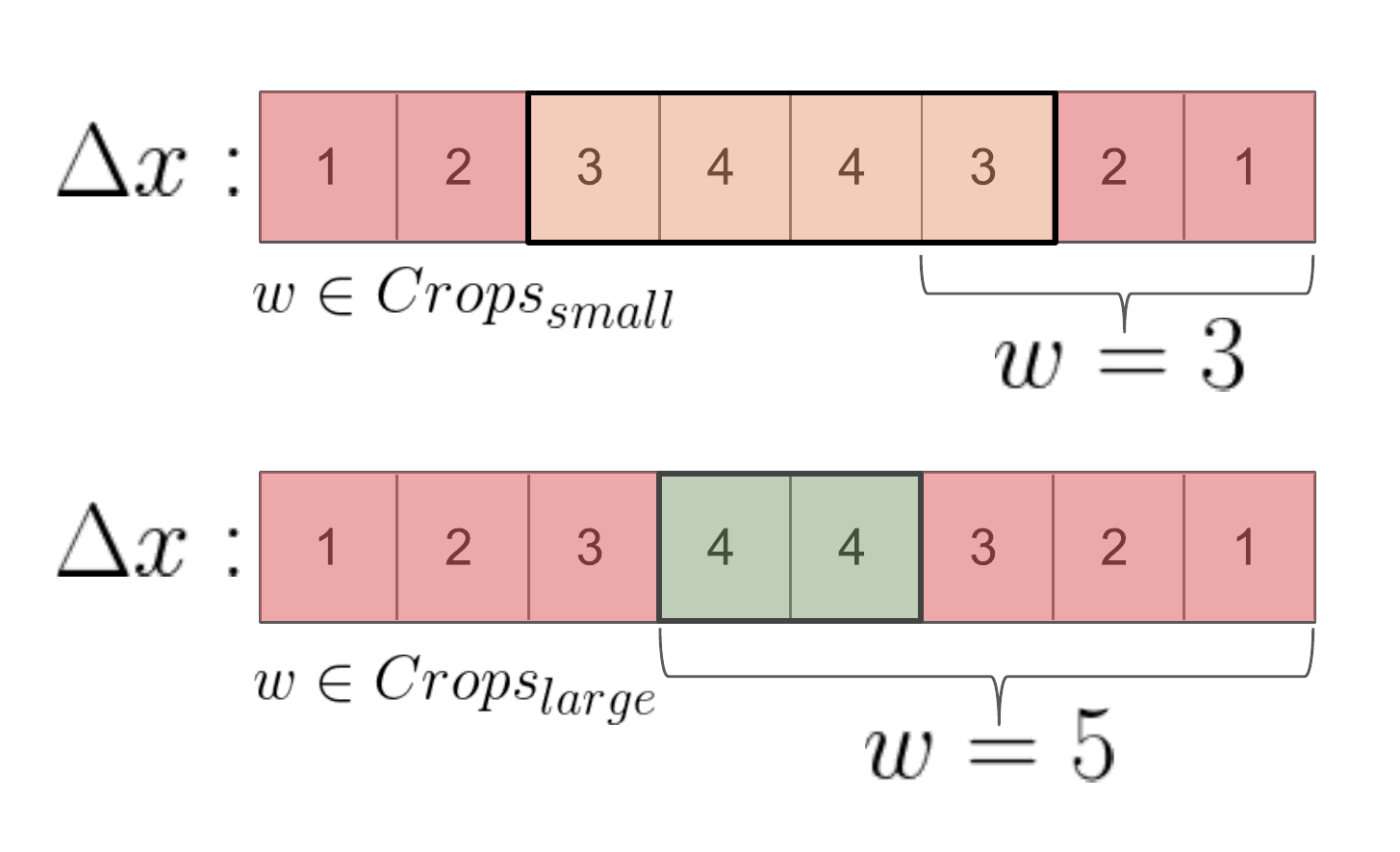}
    \caption{{\bf Illustration of Lemma 1}. The probability for a pixel to be sampled within a valid random crop location, depends on the image width $W$, the crop width $w$ and the distance to the closest border $\Delta x$. For both samples $W=8$. Top: $w=3$ ($w \in Crops_{small}$). Bottom: $w=5$ ($w \in Crops_{large}$). Each pixel is labeled with $\Delta x$.}
\label{fig:1dprob}
\end{figure}

\label{method:scene_scoping}
Due to the limited number of labeled examples available for optical flow supervised learning, most of the leading methods, in both supervised and unsupervised learning, are using cropping of a \textbf{fixed sized} randomly located patches. We are interested in understanding the chances of a pixel to be sampled, within a randomly located fixed size crop, as a function of its location in the image.

\noindent \textbf{1D image random cropping statistics\quad}
Consider a 1D image with a width $W$, a crop size $w$ and a pixel location $x$.
Let $\Delta x$ denote the distance of the pixel from the closest border, and $\Delta w$ denote the difference between the image width $W$ and the crop size $w$. Let $Crops_{large}$ be the set of crop sizes with $w$ larger than half of the width, $\frac{W}{2} < w\leq W$. Let $Crops_{small}$ be the complement set of crop sizes smaller or equal to half of the width, $0 < w\leq \frac{W}{2}$. Two instances of this setup are depicted in Fig.~\ref{fig:1dprob}.

Using the notations above, pixels are separated into three different categories, described in the following lemma. 
\begin{lemma}
For an image size $W$ and a randomly chosen crop of size $0<w \leq W$ the probability of a pixel, with coordinate $x$ and distance to the closest border $\Delta x$ to be sampled, is as follows:
\begin{equation}
    P(x| W,w) =
    \begin{cases*}
      1 & if $\Delta w<\Delta x$  \\ 
      \frac{w}{\Delta w + 1} & if $ w\leq\Delta x$ \\ 
      \frac{\Delta x}{\Delta w + 1} & otherwise 
    \end{cases*}
\label{casesprob}
\end{equation}
where $\Delta w + 1$ is the number of valid crops. 
\end{lemma}
\begin{proof}
For illustration purposes, the three cases are color coded, respectively, as green, orange, and red, in Fig.\ref{fig:1dprob}. We handle each case separately.
(i) \textbf{Green}: Every valid placement must leave out up to $\Delta w$ pixels from the left or the right. Since $\Delta x$ is larger than $\Delta w$, the pixel $x$ must be covered. (ii) \textbf{Orange}: In this case, there are $w$ possible locations for pixel $x$ within the patch, all of which are achievable, since $\Delta x$ is large enough. Therefore, $w$ patch locations out of all possible patches contain pixel $x$. (iii) \textbf{Red}: In this case, the patch can start at any displacement from the edge that is nearer to $x$, that is not larger than $\Delta x$, and still cover pixel $x$. Therefore, there are exactly $\Delta x$ locations out of the $\Delta w + 1$ possible locations.
\end{proof}

\newpage
\noindent \textbf{2D image random cropping statistics\quad}
Since the common practice in current state-of-the-art optical flow training protocols is to crop a fixed sized crop, in the range
$\frac{W}{2}<w\leq W$ \((w \in Crops_{large})\),
we will focus in the reminder of this section on the green and red categories, which are the relevant categories for crop sizes with each dimension [h,w] larger than half of the corresponding image dimension (\ie in $Crops_{large}$), and represent a cropping of more than a quarter of the image.

From the symmetry of our 1D random cropping setup, in both $x$ and $y$ axes, we can use Eq.~\ref{casesprob} in order to calculate the probability of sampling pixels in a 2D image of size $[H,W]$, with a randomly located crop of a fixed size $[h,w]$. The probability of sampling a central (green) pixel remains 1, while the probability of sampling a marginal (red) pixel $(x,y)$ in 2D, is given by:
\begin{equation}
\label{eq:2d_margin_prob}
P_{red}(x,y | H,h,W,w) = \frac{\min{(\Delta x, \Delta w)}\min{(\Delta y, \Delta h)}}{(\Delta w + 1)(\Delta h + 1)}
\end{equation}
Where $\Delta h = H-h$, $\Delta w = W-w$ the difference between the image and the crop width and height, and $\Delta x, \Delta y$ represent the distance from the closest border, as before. Eq.~\ref{eq:2d_margin_prob} represents the ratio between the number of crop locations where a (marginal) pixel with $\Delta x, \Delta y$ is sampled to the number of all unique valid crop locations. An illustration of this sampling probability is demonstrated in Fig.~\ref{fig:pixel-prob-per-crop-size} for varying ratios of crop size axes and image axes.

\noindent \textbf{Fixed crop size sampling bias\quad}
As in the 1D cropping setup, given an image of size $[H,W]$ and a crop size $[h,w]$, we can define a central area (equivalent to the green pixels in 1D), which will always be sampled. Respectively, we can define a marginal area (equivalent to the red pixels in 1D), where Eq.~\ref{eq:2d_margin_prob} holds.

Analyzing Eq.~\ref{eq:2d_margin_prob} we can infer the following: (i) in the marginal area, for a fixed crop size $[h,w]$, the probability of being sampled decreases quadratically along the image diagonal, when $\Delta x$ and $\Delta y$ both decrease together, and (ii) in the marginal area, for a fixed pixel, the probability of being sampled decreases quadratically when the crop size decreases (when $\Delta w$ and $\Delta h$ both decrease together).

Therefore, when using a fixed sized crop to augment a dataset with a random localization approach, there will be a dramatic sampling bias towards pixels in the center of the image, preserved by the symmetric range of random affine parameters. For example, with the current common cropping approach for the MPI Sintel data-set, the probability of the upper left corner pixel to be sampled in a crop equals \(\frac{1}{(1024-768 + 1)(436 - 384 + 1)} = 0.000073\%\), while the pixels in the central $[332,512]$ crop will be sampled in any randomized crop location.

This sampling bias could have a sizable influence on the training procedure. Fig.~\ref{fig:pixel-prob-per-crop-size} illustrates the distribution of fast pixels in both MPI Sintel and KITTI datasets. Noticeably, pixels of fast moving objects (with speed larger than 40 pixels per frame) are often located at the marginal area, while slower pixels are more centered in the scene. This should not be a surprise, since (i)  lower pixels belong to nearer scene objects and thus have a larger scale and velocity, and (ii) fast moving objects usually cross the image borders.

Moreover, many occluded pixels are also located close to the image borders. Therefore, increasing a crop size could also help to observe a more representative set of occlusions during training. Therefore, we hypothesized that larger crops can also improve the ability to infer occluded pixels motion from the context. 

\subsection{Scene scoping approaches}
\label{method:scoping_approaches}
Fig.~\ref{fig:pixel-prob-per-crop-size} shows the crop size effect on the probability to sample different motion categories. Clearly, the category of fast pixels suffers the most from reduction of the crop sizes. We tested four different strategies for cropping the scene dynamically (per mini batch) during training:
(S1) fixed partial crop size (the common approach), (S2) cropping the largest valid crop size, {\color{black}(S3) randomizing crop size ratios from a pre-defined set with:
    \begin{subequations}
        \begin{equation}
            \label{eq:rcrop_size_fixed}
            R_{fixed} = \{(0.73,0.69),(0.84,0.86),(1,1)\}
        \end{equation}
        \begin{equation}
                \label{eq:rcrop_size_set}
                (r_{h},r_{w})= 
                randchoice(R_{fixed}),
        \end{equation}
    \end{subequations} where $(r_{h},r_{w})$ are one of the three crop ratios, 
    and strategy (S4) is a range-based crop size randomization:
    \begin{equation}
        \label{eq:rcrop_size_range}
        s=randint(round(r_{min} \cdot S), round(r_{max} \cdot S)),
    \end{equation}
    where s is the crop axis size (h or w), S is the full image axis size (H or W), and [$r_{min}$, $r_{max}$] is the range of crop size ratios $\frac{s}{S}$ for sampling.
}

We also employ different affine transformations, and dynamically change the zooming range along the training, to enlarge the set of possible scene scopes, and improve the robustness of features to different scales. In Sec.~\ref{experiments:full_train} we describe the experiments done in order to find an appropriate approach for feeding the network with a diversity of scene scopes and reducing the inherent sampling bias explained in this section, caused by the current cropping mechanisms.

\begin{figure*}
\begin{center}
  \includegraphics[width=0.99\linewidth]{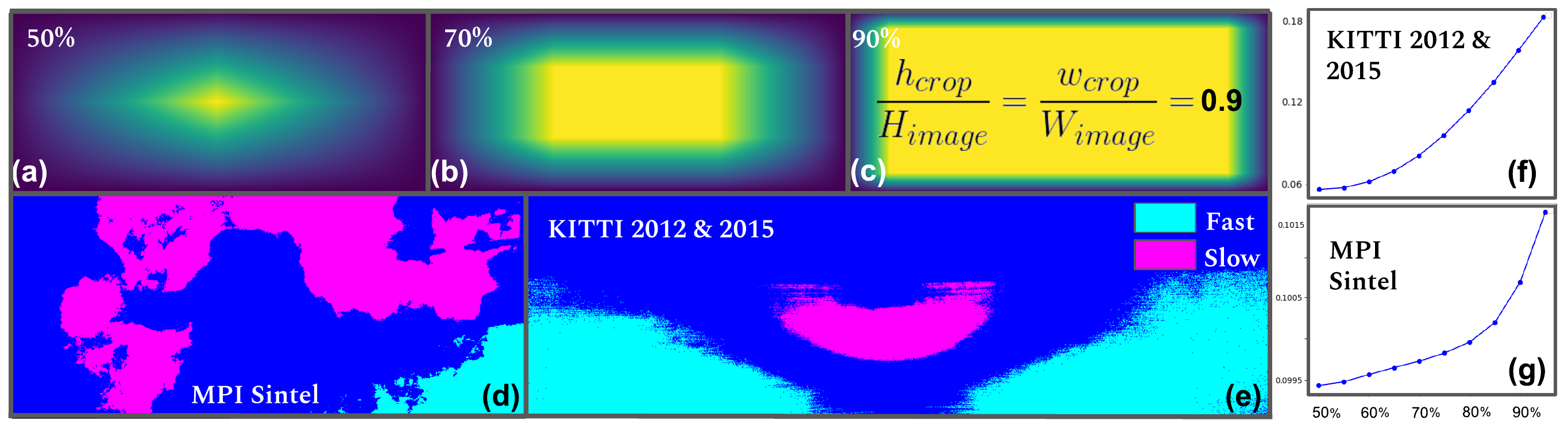}\vspace{-.4cm}
\end{center}
   \caption{\textbf{Sampling bias caused by fixed size random crops.} {\color{black}(a),(b),(c): Pixel sampling probability maps for a fixed sized crop, with ratios of 50\%, 70\% and 90\% respectively, for each axis. The probability to sample a marginal pixel shrinks drastically with the crop size. (d),(e): areas with strong prevalence for motion categories. High velocities tend to start from lower corners, while small ones tend to occur in the middle and upper part of the scene. (f),(g): graphs of the changing ratio of sampling probabilities between fast ($>40$) and slow ($<10$) pixels, for different crop and image axes ratios. clearly, fast pixels benefit more when increasing the crop size than slow pixels.}}
\label{fig:pixel-prob-per-crop-size}
\end{figure*}

\begin{figure}
    \begin{center}
    \includegraphics[width=0.98\linewidth]{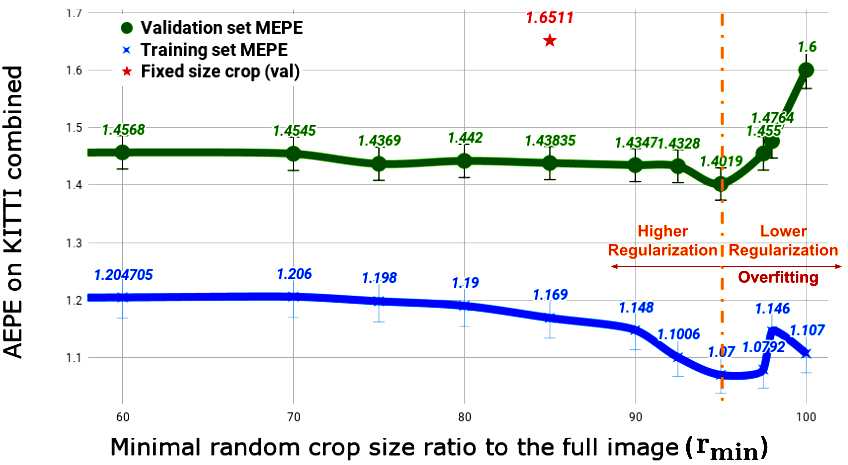}
    \end{center}
    \caption{Accuracy of models trained with different ranges of random crop sizes, on the combined KITTI dataset. The maximal crop size is the full image. Validation AEPEs improve when increasing the minimal crop size ratio ($r_{min}$ in Eq.~\ref{eq:rcrop_size_range}) up to 95\% of the full image axes. AEPE for a fixed sized crop based training: {\color{red}$\ast$}.}
\label{fig:min_crop_comparison}
\end{figure}

In addition to testing the scope modifications based on our analysis, we were also interested in testing different parameters of the training.

\section{Training, regularization and augmentation}
\label{method:training}

\noindent \textbf{Learning rate and training schedules\quad}
\label{method:lr_sched}
The common LR schedules, proposed by Ilg~\etal~\cite{IMSKDB17}, used to train deep optical flow networks, are $S_{long}$ or $S_{short}$ for the pretraining stage, and $S_{ft}$ for the finetune phases. We used the shorter schedule, suggested by~\cite{Hur2019CVPR}, of using $S_{short}$ for pretraining, half of $S_{short}$ for FThings finetuning, and $S_{ft}$ for Sintel and KITTI datasets. We also examine the effect of retraining and over-training specific stages {\color{black}of the multi-phase training.}

\noindent \textbf{Data augmentation\quad}
\label{method:augmentation}
The common practice in the current training protocol employs two types of data augmentation techniques: photometric and geometric. The details of these augmentations did not change much since FlowNet~\cite{7410673}. The photometric transformations include input image perturbation, such as color and gamma corrections. The geometric augmentations include a global or relative affine transformation, followed by random horizontal and vertical flipping. Due to the spatial symmetric nature of the translation, rotation and flipping parameters, we decided to focus on the effect of zooming changes, followed by our new cropping approaches.

\noindent \textbf{Regularization\quad}
\label{method:reg}
The common protocol also includes weight decay and adding random Gaussian noise to the augmented image. In our experiments, we tested the effect of eliminating these sources of regularization at different stages of the multi-phase training.

\section{Experiments}
\label{experiments}

\begin{small}
\begin{table}[]
     \resizebox{\linewidth}{!}{%
    \centering
    \begin{tabular}{@{}c@{~}c@{~}c@{~}c@{~}|@{~}c@{~}c@{~}c@{~}c@{~}c@{}}
    \hline
    {\begin{tabular}[c]{@{}c@{}}RD+\\ RN\end{tabular}} & {\begin{tabular}[c]{@{}c@{}}Zoom\\ changes\end{tabular}} & {\begin{tabular}[c]{@{}c@{}}Max\\ crop\end{tabular}} & {\begin{tabular}[c]{@{}c@{}}Random\\ crop\end{tabular}} & {\begin{tabular}[c]{@{}c@{}}Sintel\\ train\end{tabular}} & {\begin{tabular}[c]{@{}c@{}}Sintel\\ val\end{tabular}} & 
    {\begin{tabular}[c]{@{}c@{}}KITTI\\ train\end{tabular}} & {\begin{tabular}[c]{@{}c@{}}KITTI\\ val\end{tabular}} & {\begin{tabular}[c]{@{}c@{}}KITTI\\ val Out\%\end{tabular}}  
    \\ \hline
    x & x & x & x & 2.660 & 3.312 & 1.728 & 1.651 & 0.057 \\
    \checkmark &  &  &  & 2.623 & 3.224 & 1.654 & 1.644 & 0.056 \\
    \checkmark & \checkmark &  &  & 2.453 & 3.108 & 1.580 & 1.649 & 0.056 \\
    \checkmark &  & \checkmark &  & 2.428 & 3.053 & 1.182 & 1.594 & 0.059 \\
    \checkmark &  &  & \checkmark & 2.537 & 3.081 & \textbf{1.070} & \textbf{1.402} & \textbf{0.051} \\
    \checkmark & \checkmark & \checkmark &  & \textbf{2.320} & 2.987 & 1.225 & 1.607 & 0.059 \\
    \checkmark & \checkmark &  & \checkmark & 2.349 & \textbf{2.971} & 1.094 & 1.434 & \textbf{0.051} \\
\hline
    \end{tabular}
     }
    \caption{\textbf{Finetuning experiments.} Results were calculated with AEPE, except outlier percentage for KITTI validation.
    RD + RN is for removing random noise and weight decay. Zoom changes include an increased zoom out and a gradual reduction of the zoom in. Max crop is for using the maximal valid crop size for a batch. Random crop is for using Eq.~\ref{eq:rcrop_size_range} when sampling the crop size. \label{tab:ablation} }
   \smallskip

\centering
    \begin{tabular}{@{}l@{~~}c@{~~}c@{~~}c@{~~}c@{}}
    \toprule
    Model & Max zoom & WD+RN & VAL MEPE & Sintel MEPE \\ \midrule
    C* & 1.5 & \checkmark & - & \textbf{3.138} \\
    C1 & 1.5 & x & 1.622 & 3.264 \\
    C2 & 1.3 & x & \textbf{1.597} & 3.321 \\ \bottomrule
    \end{tabular}
    \caption{\textbf{Removing regularization in pretraining.} Models trained 108 epochs, from initialized weights, without weight decay and random noise, for two maximal zoom values, on FChairsOcc.\label{tab:chairs_training}} 
    \smallskip

\centering
    \begin{tabular}{lccccc}
    \toprule
    {\begin{tabular}[c]{@{}l@{}}Things\\ Model\end{tabular}} & {\begin{tabular}[c]{@{}c@{}}WD+\\ RN\end{tabular}} & {\begin{tabular}[c]{@{}c@{}}Start\\ From\end{tabular}} & {Epochs} & {\begin{tabular}[c]{@{}c@{}}Val\\ MEPE\end{tabular}} & {\begin{tabular}[c]{@{}c@{}}Sintel\\ MEPE\end{tabular}} \\ \midrule
    T2 & \checkmark & C* & 109-159 & 1.843 & 2.613 \\
    T3 & \checkmark & C1 & 109-159 & 1.829 & \textbf{2.544} \\
    T4 & x & T3 & 159-165 & \textbf{1.817} & 2.545 \\ \bottomrule
    \end{tabular}
    \caption{The negative effect of over-training and reducing regularization on early stages. T3 and T4 were trained with larger scopes.\label{tab:things_over_train}}
    \smallskip
\centering
    \begin{tabular}{lccc}
    \toprule
    KITTI model & Start from & Val MEPE & Outliers \\ \midrule
    T2\_K & T2 & \textbf{1.474} & 0.054 \\
    T3\_K & T3 & 1.475 & \textbf{0.053} \\ \bottomrule
    \end{tabular}
    \caption{Higher gains in early stages do not always translate to fine-tune gains. All models trained on KITTI combined.
    \label{tab:kitti_pre_train_effect}}
    \smallskip
\centering
    \begin{tabular}{lccc}
    \toprule
    Sintel model & WD+RN & Val MEPE & Val OCC F1 \\ \midrule
    T3\_SC1 & \checkmark & 2.119 & 0.700 \\
    T3\_SC2 & x & \textbf{2.108} & \textbf{0.703} \\ \bottomrule
    \end{tabular}
    \caption{The positive effect of reducing regularization in finetune. All models trained on Sintel combined, from T3.\label{tab:sintel_finetune_reg}}
\end{table}
\end{small}

In this section, we describe the experiments and results for our research questions. Specifically, we tested (i) how can we change the current training pipeline in order to improve the final accuracy, and (ii) the effect of feeding the network with different scopes of the input during training, using different cropping and zooming changes.

All our experiments on KITTI used both KITTI2012 and KITTI2015, and for Sintel both the clean and final pass, for training and validation sets. We denote the combined datasets of Sintel and KITTI as Sintel combined and KITTI combined.
We also tested the approach, suggested by Sun~\etal~\cite{Sun2018PWC-Net}, to first train on a combined Sintel dataset, followed by another finetune on the final pass.

All of our experiments employ the common End Point Error metric for flow evaluation, and F1 for occlusion evaluation. KITTI experiments also present outlier percentage.

\subsection{Finetuning a pretrained model}
\label{experiments:finetune}

\begin{small}
\begin{table}[]
    \centering
    \begin{tabular}{@{}l@{~}c@{~}c@{}}
    \hline
    {Method} & \multicolumn{1}{l}{{MEPE}} & {\begin{tabular}[c]{@{}c@{}}Outlier\\ \%(EPE \textgreater{}3 px)\end{tabular}} \\ \hline
    \#1: FP (320,896) & 1.651 & 0.057 \\
    \#2: FF (370,1224) & 1.594 & 0.059 \\
    \#4: RR {[}0.75,0.9{]} & 1.472 & 0.052 \\
    \#3: FR \{(0.73,0.69),(0.84,0.86),(1,1)\} & 1.466 & 0.053 \\
    \#4: RR {[}0.9,1{]} & 1.435 & 0.053 \\
    \#4: RR {[}0.95,1{]} & \textbf{1.402} & \textbf{0.051} \\ \hline
    \multicolumn{3}{l}{{Re-finetune}} \\ \hline
    \#2:  FF (370,1224) -\textgreater \#4: RR {[}0.95,1{]} & 1.421 & 0.052 \\
    \#4: RR {[}0.95,1{]} -\textgreater \#4: RR {[}0.95,1{]} & 1.393 & \textbf{0.051} \\
    \#4: RR {[}0.9,1{]} -\textgreater \#4: RR {[}0.95,1{]} & \textbf{1.377} & \textbf{0.051} \\ \hline
    \end{tabular}
    \caption{\textbf{Random cropping experiments.} Ranges are specified in $[]$, sets in $\{\}$, and fixed sizes in $()$. FP is fixed partial, FF is fixed full, RR is random range (Eq.~\ref{eq:rcrop_size_range}), FR is fixed range (Eq.~\ref{eq:rcrop_size_fixed}). Up: best results from each method described in Sec.~\ref{method:scoping_approaches} (the method number is on the left). Bottom: retraining experiments show that it is better to train a more regularized model in the first KITTI finetune, although it gets a lower MEPE on the first finetune.\label{tab:randcrop}
}
\smallskip
    \centering
    
    \begin{tabular} 
    {@{}l@{~}c@{~}c@{~}c@{~}c@{}}
    
    \toprule
    Method & Clean & Final & Mean & Type \\
    \midrule
    FlowNet2~\cite{IMSKDB17} & 0.377 & 0.348 & 0.362 & fwd-bwd \\
    MirrorFlow & 0.39 & 0.348 & 0.369 & CNN \\
    S2DFlow & 0.47 & 0.403 & 0.436 & CNN \\
    ContinualFlow~\cite{Neoral2018ACCV} & - & - & 0.48 & CNN \\
    SelFlow~\cite{Liu:2019:SelFlow} & 0.59 & 0.52 & 0.555 & fwd-bwd \\
    FlowDispOccBoundary~\cite{ISKB18} & 0.703 & 0.654 & 0.678 & CNN \\
    IRR-PWC~\cite{Hur2019CVPR} & 0.712 & 0.669 & 0.690 & CNN \\
    ScopeFlow (Ours) & \textbf{0.740} & {\textbf{0.711}} & \textbf{0.725} & CNN \\ \bottomrule
    \end{tabular}
    \caption{\textbf{Occlusion estimation comparison on Sintel.} Results were calculated with F1 score (higher is better).
    \label{tab:occlusion_results}}
\end{table}

\begin{table*}[t]
    \resizebox{\textwidth}{!}{%
    \begin{tabular}{@{}l@{~}c@{~}c@{~}c@{~}c@{~}c@{~}c@{~}c@{~}c@{~}c@{~}c@{~}c@{}}
    \toprule
    \multirow{2}{*}
     {{Method}} & 
       \multicolumn{7}{c}{{Sintel -- final pass}} & \multicolumn{1}{c}{{clean pass}} & \multicolumn{2}{c}{{KITTI -- 2012}} & \multicolumn{1}{c}{{2015}}   \\
      \cmidrule(lr){2-8}
      \cmidrule(lr){9-9}
     \cmidrule(lr){10-11}
          \cmidrule(lr){12-12}
      & 
      \multicolumn{1}{c}{all} & \multicolumn{1}{c}{matched} & \multicolumn{1}{l}{unmatched} & \multicolumn{1}{c}{d0-10} & \multicolumn{1}{l}{d60-140} & \multicolumn{1}{c}{s0-10} & \multicolumn{1}{l}{s40+} & \multicolumn{1}{c}{all} & \multicolumn{1}{c}{Out-All} & \multicolumn{1}{c}{Avg-All} & \multicolumn{1}{c}{Fl-all}   \\ \midrule
    FlowNet2 \cite{IMSKDB17} & 5.739 & 2.752 & 30.108 & 4.818 & 1.735 & 0.959 & 35.538 & 3.959 & 8.8 & 1.8 & 10.41 \\
     MR-Flow\cite{Wulff:CVPR:2017} & 5.376 & 2.818 & 26.235 & 5.109 & 1.755 & 0.908 & 32.221 & \textbf{2.527} & - & - & 12.19 \\
 DCFlow\cite{xu_2017_dcflow} & 5.119 & 2.283 & 28.228 & 4.665 & 1.44 & 1.052 & 29.351 & 3.537 & - & - & 14.86\\
    PWC-Net\cite{Sun2018PWC-Net}* & 5.042 & 2.445 & 26.221 & 4.636 & 1.475 & 2.986 & 31.07 & 4.386 & 8.1 & 1.7 & 9.6\\
    ProFlow\cite{Maurer2018ProFlowLT} & 5.017 & 2.596 & 24.736 & 5.016 & 1.601 & 0.91 & 30.715 & 2.709 & 7.88 & 2.1 & 15.04\\
    LiteFlowNet2\cite{hui19liteflownet2} & 4.903 & 2.346 & 25.769 & 4.142 & 1.546 & 0.797 & 31.039 & 3.187 & 6.16 & 1.4 & 7.62\\
    PWC-Net+\cite{Sun2018:Model:Training:Flow}*+ & 4.596 & 2.254 & 23.696 & 4.781 & 1.234 & 2.978 & 26.62 & 3.454 & 6.72 & 1.4 & 7.72\\
    HD3-Flow\cite{yin2019hd3} & 4.666 & 2.174 & 24.994 & \textbf{3.786} & 1.647 & 0.657 & 30.579 & 4.788 & \textbf{5.41} & 1.4 & 6.55\\
    IRR-PWC\cite{Hur2019CVPR}*\textasciicircum{} & 4.579 & 2.154 & 24.355 & 4.165 & 1.292 & 0.709 & 28.998 & 3.844 & 6.7 & 1.6 & 7.65\\
    MFF\cite{ren2018fusion}*+ & 4.566 & 2.216 & 23.732 & 4.664 & 1.222 & 0.893 & 26.81 & 3.423 & 7.87 & 1.7 & 7.17\\
    ContinualFlow\_ROB\cite{Neoral2018ACCV}*+ & 4.528 & 2.723 & \textbf{19.248} & 5.05 & 1.713 & 0.872 & 26.063 & 3.341 & - & - & 10.03\\
    VCN\cite{yang2019vcn} & 4.52 & 2.195 & 23.478 & 4.423 & 1.357 & 0.934 & 26.434 & 2.891 & - & - & \textbf{6.3}\\
    SelFlow\cite{Liu:2019:SelFlow}* & 4.262 & 2.04 & 22.369 & 4.083 & 1.287 & \textbf{0.582} & 27.154 & 3.745 & 6.19 & 1.5 & 8.42\\
    ScopeFlow*, regularization & 4.503 & 2.16 & 23.607 & 4.124 & 1.292 & 0.706 & 27.831 & 3.86 & - & - & - \\
    ScopeFlow*, zooming & 4.317 & 2.086 & 22.511 & 4.018 & 1.311 & 0.739 & 26.218 & 3.696 & - & - & -\\
    ScopeFlow* (Ours) & \textbf{4.098} & \textbf{1.999} & 21.214 & 4.028 & \textbf{1.18} & 0.725 & \textbf{24.477} & 3.592 & 5.66 & \textbf{1.3} & 6.82 \\
    \bottomrule
    \end{tabular}
    }
    \caption{\textbf{Public benchmarks results.} Models with comparable architecture (PWC-Net) are marked with *. Models using extra data in finetune are marked with +. Our baseline model is marked with \textasciicircum{}. We get the best EPE results in both Sintel and KITTI2012 benchmarks, surpassing all other comparable variants of our baseline model on KITTI2015, with a considerable improvement to our baseline.}
    \label{tab:results_sintel_kitti}
\end{table*}
\end{small}

Since the cost of pretraining is approximately $\times 7$ than of the final finetune, we first present experiments done on the finetuning phase, in which we employ models pretrained on FChairs and FThings, published by the authors of IRR-PWC. {\color{black}These experiments are conducted on the Sintel dataset, since it has similar statistics of displacements to the FChairs dataset~\cite{Mayer_2018} used for pretraining}. We tested different
training protocol changes, and found that substantial gains could be achieved using the following changes:
\noindent \textbf{1. Cropping strategies.} 
During the initial finetune, we tested the cropping approaches specified in Sec.~\ref{method:scoping_approaches} on Sintel. The results specified in Tab.~\ref{tab:ablation} show that the range-based crop size randomization approach (Eq.~\ref{eq:rcrop_size_range}) was comparable to taking the maximal valid crop (although much more efficient computationally), and both improved Sintel validation error of models trained with smaller fixed crop sizes.

\noindent \textbf{2. Zooming strategies.} 
We found that applying a new random zooming range of $[0.8,1.5]$ alone, which increases the zoom out, and gradually reducing the zoom in to $1.3$, achieved considerable gains for Sintel in all evaluation parameters, with and without cropping strategy changes. Interestingly, increasing the zoom out range without any change to the crop size provided 50\% of this gain. We suggest that this is additional evidence for the existing bias in small crop sizes, as explained in Sec.~\ref{method:scene_scoping}. 

\noindent \textbf{3. Removing artificial regularization.} 
Removing the random noise and weight decay helped us to achieve extra 2\%-3\% of improvement during Sintel finetune, demonstrating the benefit of reducing augmentation in advanced stages.

\subsection{Applying changes to the full training procedure}
\label{experiments:full_train}
We then tested {\color{black}the changes from Sec.~\ref{experiments:finetune}, along with all four cropping approaches described in Sec.~\ref{method:scoping_approaches},} on the different stages of the common curriculum learning pipeline. Since we wanted to test our training changes and compare our results to other variants of the baseline architecture, we decided not to use any other dataset, other than the common pretraining or benchmarking datasets. For FChairs and FThings, all trained models were evaluated on Sintel training images, as an external test set.

\noindent \textbf{FChairs pretraining\quad} For pretraining, we downloaded the newly released version of FChairs~\cite{Hur2019CVPR}, which includes occlusions ground truth annotations. We trained two versions of the IRR-PWC model on FChairsOCC, for 108 epochs on 4 GPUs: (i) C1: removing weight decay and random noise (ii) C2: same as (i) with reduced zoom in. We then evaluated both models and the original model, trained by the authors with weight decay and original zoom in of $1.5$, denoted by C*. Results are depicted in Tab.~\ref{tab:chairs_training}, showing that regularization is important in early stages, since removing either weight decay and random noise, or reducing the zoom-in hurt the performance. 

\noindent \textbf{FThings finetune\quad} We then trained three versions of IRR-PWC on FThings, for 50 epochs: (i) T2: resuming C* training with batch size of 2, with the original crop size of $[384,768]$, (ii) T3: resuming C1 with the maximal crop size, and (iii) T4: resuming T3 without weight decay and random noise. We can infer from the results in Tab.~\ref{tab:things_over_train}: (i) increasing the scope during FChairs training leads to better accuracy on the Sintel test set, and (ii) over-training without weight decay and random noise did not improve the results on the external test set (but did on the validation).

\noindent \textbf{KITTI finetune\quad} We trained two different versions, both with the same protocol, for 550 epochs on KITTI combined: (i) resuming T2 and (ii) resuming T3. Although T3 got better performance in the evaluation, after finetuning, both results were similar on KITTI validation, as shown in Tab.~\ref{tab:kitti_pre_train_effect}.

\noindent \textbf{Sintel finetune\quad} Two different versions were trained with the same protocol, for 290 epochs on Sintel combined, both from T3: (i) with weight decay and random noise and (ii) without weight decay and random noise. The results, presented in Tab.~\ref{tab:sintel_finetune_reg}, show that reducing regularization in Sintel finetune produced an extra gain.

\begin{figure*}[t]
\centering
  \includegraphics[width=0.99\textwidth]{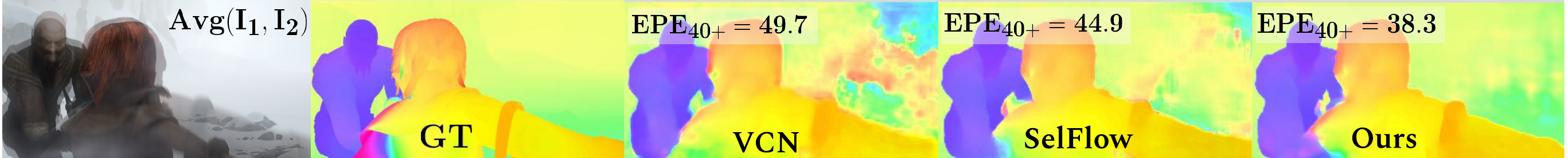}

   \caption{\textbf{Improving estimation for fast moving pixels.} {\color{black}A qualitative comparison with the other two leading methods on the Sintel benchmark. Images were downloaded from MPI Sintel website, evaluated online on a test image, for the category of fast pixels (40+). Left to right: averaged first and second image and flow visualization for each method. $EPE_{40+}$ is the end point error calculated on fast pixels.}}
\label{fig:occ-fast-improve}
\end{figure*}

\noindent \textbf{Dynamic scene scoping\quad}
Since the scoping approaches were already tested on Sintel during the initial finetune, we further tested the four different approaches for dynamic scene scoping, detailed in Sec.~\ref{method:scoping_approaches}, on the combined KITTI dataset. The results are depicted in Tab.~\ref{tab:randcrop}. For KITTI, cropping the maximal valid crop per batch shows noticeable improvement from using a fixed sized crop. However, for KITTI datasets, strategy S4 (Eq.~\ref{eq:rcrop_size_range}) shows much better performance than using the maximal valid crop size. In order to find the optimal range of crop size ratios (Eq.~\ref{eq:rcrop_size_range}), we trained different models with different ranges of crop size to image ratios $[r_{min},r_{max}]$. All models used an upper crop size ratio limit $r_{max}$ equal to 1 (\ie the maximal valid crop for the batch), and different lower limit $r_{min}$, ranging from $0.5$ to $1$ and representing random crop sizes with different aspect ratios, which are larger than a quarter of the image.

Fig.~\ref{fig:min_crop_comparison} shows the training and validation accuracy as a function of the lower ratio of the range of randomized crop sizes. Specifically, the best results obtained with $r_{min}$ equals the 0.95, as also demonstrated in Fig.~\ref{fig:min_crop_comparison}. The validation accuracy improves consistently when increasing $r_{min}$ from 0.5 to 0.95 and then starting to deteriorate until $r_{min}$ is reaching the maximal valid crop size. As can be seen, when enlarging the crop size expectation, we also reduce the regularization provided by the larger number of scopes (as analyzed in Eq.~\ref{eq:2d_margin_prob}). This observation can be considered as additional evidence of a regularization-accuracy trade-off in the training process. It also emphasizes the power of Eq.~\ref{eq:rcrop_size_range}, in improving the training outcome, while keeping the regularization provided by partial cropping.

\noindent \textbf{Re-finetune with dynamic scene scoping\quad}
In order to further understand the effect of this regularization-accuracy trade-off, we re-trained three models with the best random approach $([r_{min},r_{max}] = [0.95,1])$ on the KITTI combined set, using the same finetuning protocol. We took three different models, finetuned with $r_{min}\in \{0.9,0.95,1\}$, as the checkpoint for this second finetune.

As described in the lower part of Tab.~\ref{tab:randcrop}, finetuning again on the KITTI dataset improved the validation accuracy for all starting points (compared to their accuracy after the first finetune). Surprisingly, in the second finetune, repeating the best approach (of randomizing using Eq.~\ref{eq:rcrop_size_range} with $([r_{min},r_{max}] = [0.95,1])$) did not provide the best result. The best approach was to finetune for the second time from a model with a larger range $([r_{min},r_{max}] = [0.9,1])$, thus stronger regularization, but lower EPE in the first finetune. We propose to consider this as additional evidence for the notion that gradually reducing regularization in optical-flow training, helps to achieve a better final local minima.

\noindent \textbf{Full training insights\quad} Concluding Sec.~\ref{experiments:full_train} experiments, we suggest the following: (i) larger scopes can improve optical flow training {\color{black}as long as the regularization provided by small crops is not needed}, {\color{black} (ii) range based crop size randomization (Eq.~\ref{eq:rcrop_size_range}) is a good strategy when regularization is needed, }(iii)  strong regularization is required on early stages, and should be relaxed when possible, and (iv) gains on early stages do not always improve the final accuracy.

\subsection{Occlusion estimation} 
We evaluated the occlusion estimation of our trained models, using the F1 score, during all stages of the full training. As demonstrated in Tab.~\ref{tab:sintel_finetune_reg}, it appears that gains in optical flow estimation are highly correlated with improvements in occlusion estimation. This might be due to the need 
for a network to identify non-matchable pixels and to infer the flow from the context.  Tab.~\ref{tab:occlusion_results} shows a comparison of our F1 score to other reported results, on the MPI Sintel dataset. Our updated training protocol improves the best reported occlusion result by more than $5\%$.

\subsection{Official Results}
\label{experiments:results}
Evaluating our best models on the MPI Sintel and KITTI benchmarks shows a clear improvement over the IRR-PWC baseline, and an advantage over all other PWC-Net variants. 

\noindent \textbf{MPI Sintel\quad}
We uploaded results for three different versions: (i) with reduced regularization, (ii) with improved zooming schedule and (iii) with the best dynamic scoping approach. As Tab.~\ref{tab:results_sintel_kitti} shows, there is a consistent improvements on the test set. This is congruent with the results in Tab.~\ref{tab:ablation}, obtained on the validation set.

At the time of submission, our method ranks first place on the MPI Sintel benchmark, improving two-frame methods by more than 10\%, surpassing other competitive methods trained on multiple datasets, with multiple frames and all other PWCNet variants, using an equal or larger size of trainable parameters. On the clean pass, we improve the IRR-PWC result by 20 ranks and 7\%. Interestingly, analyzing Sintel categories in Tab.~\ref{tab:results_sintel_kitti}, our model is leading in the categories of fast pixels ($S40_{+}$) and non-occluded pixels, while also producing the best estimation for occluded pixels among two frame methods. This is consistent with our insights on these challenging categories from Sec.~\ref{method:scene_scoping}. Fig.~\ref{fig:occ-fast-improve} shows a comparison of our method in the category of fast pixels, with the other two leading methods on Sintel.

\noindent \textbf{KITTI\quad}On KITTI 2012 and KITTI 2015, we saw a consistent improvement from the baseline model results, of more than 19.7\% and 12\% respectively,  surpassing all other published methods of the popular PWC-Net architecture, and achieving state-of-the art EPE results among two frame methods. Since our training protocol can be readily applied to other methods, we plan, as future work, to test it on other leading architectures.

\section{Conclusions}
\label{conclusions}

While a lot of effort is dedicated for finding effective network architectures, much less attention is provided to the training protocol. However, when performing complex multi-phase training, there are many choices to be made, and it is important to understand, for example, the proper way to schedule the multiple forms of network regularization. In addition, the method used for sampling as part of the augmentation process can bias the training protocol toward specific types of samples.

In this work, we show that the conventional scope sampling method leads to the neglect of many challenging samples, which hurts performance. We advocate for the use of larger scopes (crops and zoom-out) when possible, and a careful crops positioning when needed.
We further show how regularization and augmentation should be relaxed as training progresses. The new protocol developed has a dramatic effect on the performance of our trained models and leads to state of the art results in a very competitive domain.

\section*{Acknowledgement}
This project has received funding from the European Research Council (ERC) under the European Unions Horizon 2020 research and innovation programme (grant ERC CoG 725974).

{\small
\bibliographystyle{ieee_fullname}
\bibliography{scopeflow}
}

\clearpage

\appendix

\maketitle

\section{Introduction}
With this appendix, we would like to provide more details on our training pipeline and framework, as well as more visualizations of the improved flow and occlusion estimation.

The ScopeFlow approach provides an improved training pipeline for optical flow models, which reduces the bias in sampling different regions of the input images while keeping the power of the regularization provided by fixed-size partial crops. Due to the sizable impact on performance that can be achieved by the improved training pipeline, we created a generic, easy to configure, training package, in order to encourage others to train state of the art models with our improved pipeline, as described in Sec.~\ref{sec:package}.


\section{Dynamic scoping}
The common pipeline of batch sampling and augmentation in optical flow training includes four stages: (i) sampling images, (ii) applying random photometric changes, (iii) applying a random affine transformation, and (iv) cropping a fixed-size randomly located patch. We propose changes for stages (iii) and (iv), by choosing the zooming parameters more carefully along with the training, and incorporating a new randomized cropping scheme, presented and extensively tested in our paper.

Fig.~\ref{fig:dynamic_soping} provides a demonstration of the ScopeFlow pipeline, which enlarges the variety of scopes presented during the data-driven process while reducing the bias towards specific categories.

\begin{figure}
    \begin{center}
      \includegraphics[width=0.95\linewidth]{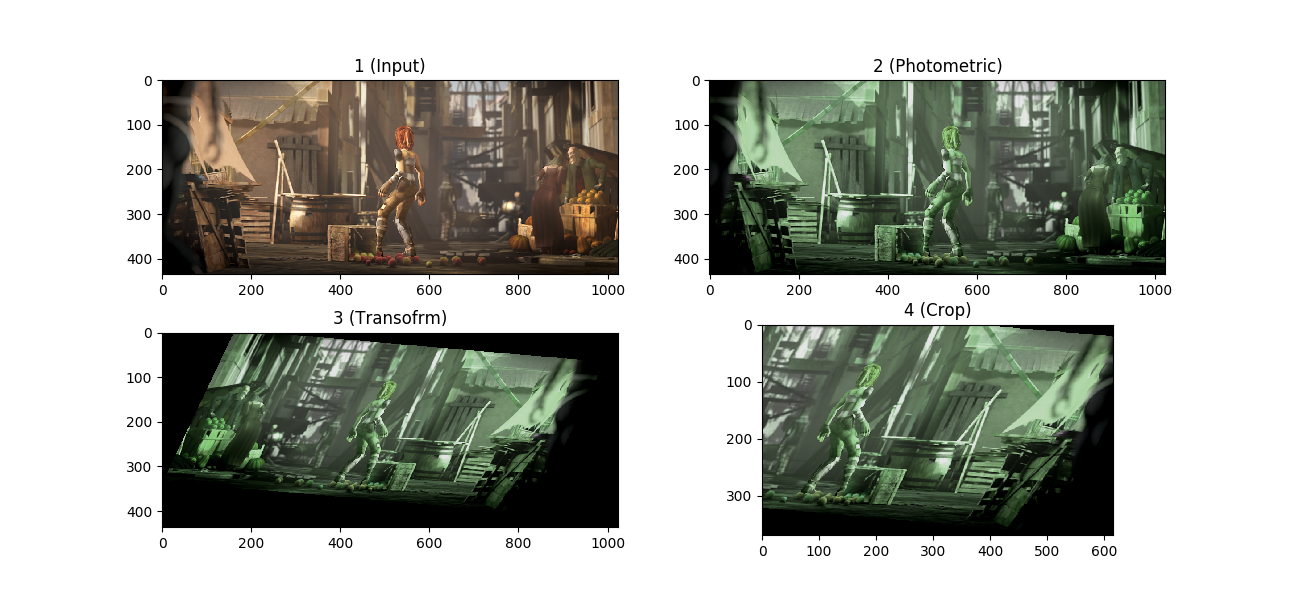}
      \includegraphics[width=0.95\linewidth]{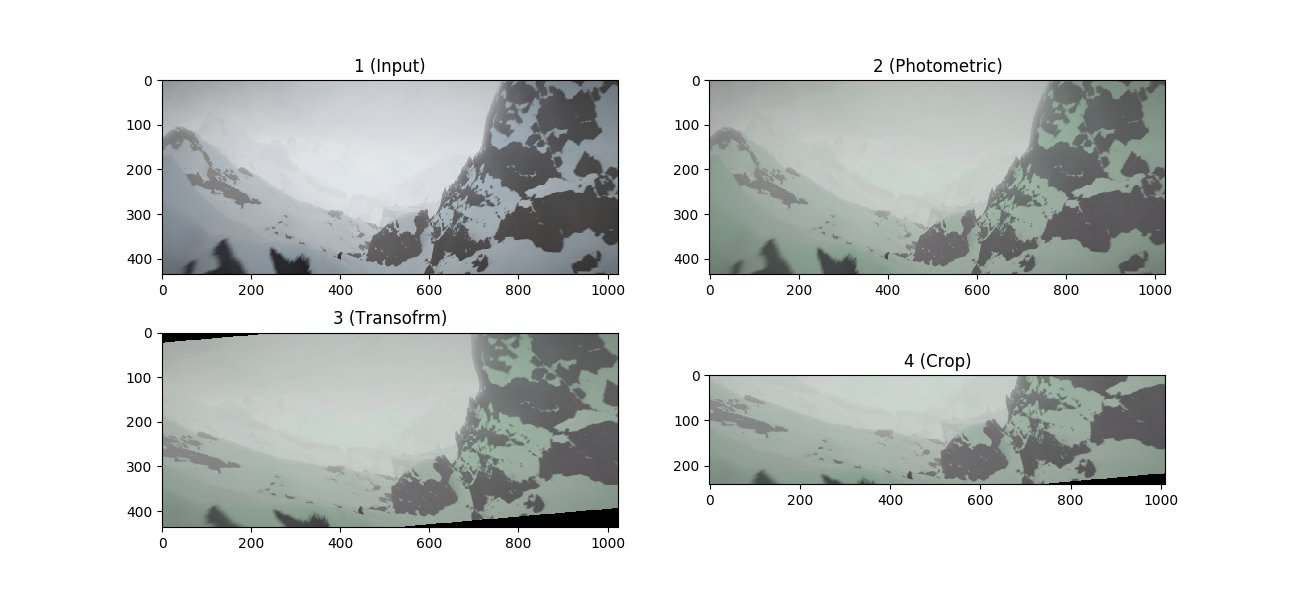}
      \includegraphics[width=0.95\linewidth]{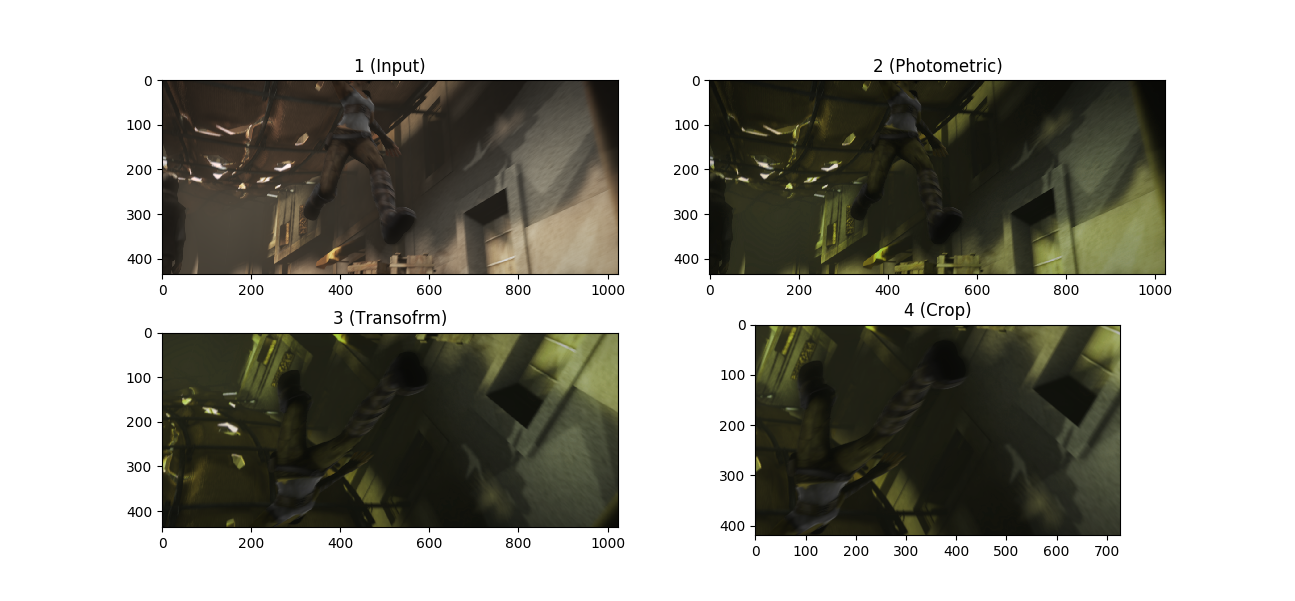}
      \includegraphics[width=0.95\linewidth]{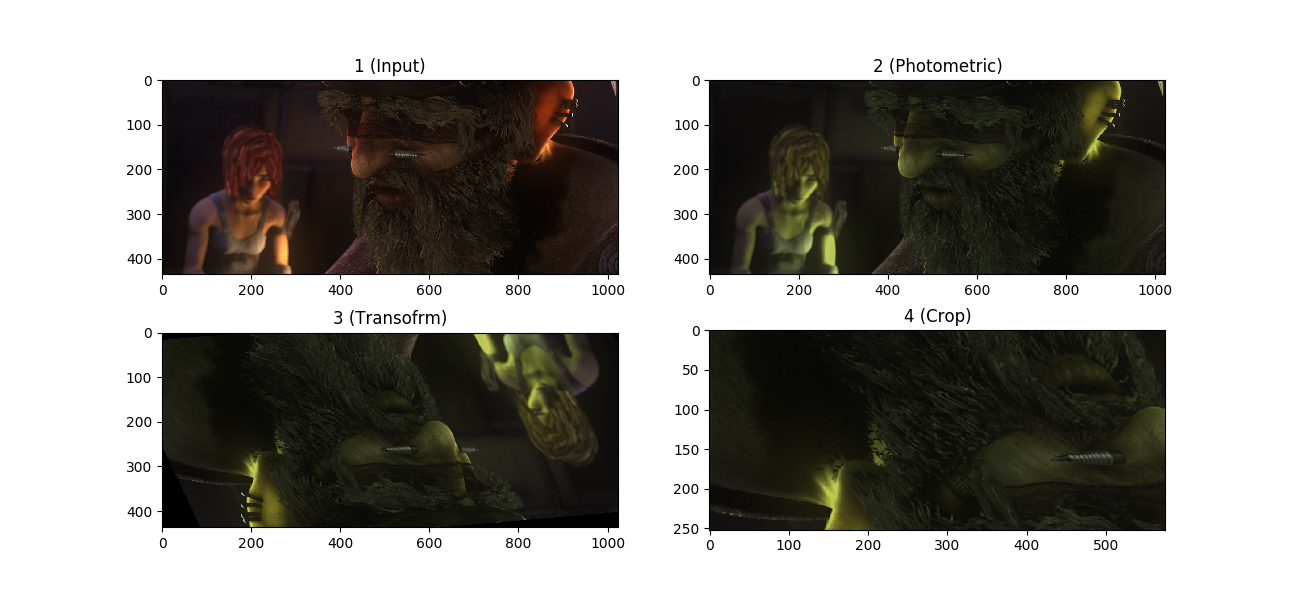}

    \end{center}
   \caption{Randomized scoping within $[r_{min},r_{max}]=[0.5,1]$. Training with ScopeFlow online-processing approach leads to the learning of richer features and reduces the error in challenging motion categories, such as fast speed and occluded pixels.}
\label{fig:dynamic_soping}
\end{figure}

\section{ScopeFlow software package}
\label{sec:package}
In order to simplify the applicability of our approach, we created a small and easy to use package, which supports YAML configurations of a multi-stage optical flow model training and evaluation. We found this approach very helpful when running experiments for finding the best scoping augmentation approach.

Our code and models are \href{https://github.com/avirambh/ScopeFlow}{\textit{available online}}.
Furthermore, we provide easy visualization of our online augmentation pipeline, as described in the README of our package.

\section{Comparison to the IRR baseline}
In our experiments, we use the IRR~\cite{Hur2019CVPR} variant, of the popular PWC-Net architecture, to evaluate our method. This variant has shown to provide excellent results, while keeping a low number of parameters. To emphasize the improvements, we give here a thorough comparison, of all the public results obtained in the main three benchmarks, for our method and the IRR baseline.

\subsection{MPI Sintel}
Other than leading the MPI Sintel~\cite{Butler:ECCV:2012} table, as can be seen in Tab.~\ref{fig:mpi_sintel_14_10} and Tab.~\ref{fig:mpi_sintel_18_11} in Sec.~\ref{sec:public_tables}, we improve the baseline IRR models by a large margin in all metrics, and in particular the challenging metrics of occlusions (14.7\%) and fast pixels (18.4\%). The only metric that did not improve is the metrics of low-speed pixels ($<40$), which should not be a surprise, since our method reduces the bias between the fast and slow pixels, as shown in our paper.

\subsection{KITTI 2012}
We uploaded our results to the KITTI 2012~\cite{Geiger2012CVPR} benchmark. As can be seen in Tab.~\ref{tab:KITTI 2012_irr} and Tab.~\ref{tab:KITTI 2012_scopeflow}, training the IRR model with ScopeFlow pipeline improves the mean EPE by more than 20\%. Moreover, the improvement is achieved for all thresholds of outliers and for all metrics.

\medskip
IRR on KITTI 2012:
\begin{table}[h!]
    \centering
    \begin{tabular}{c | c | c | c | c}
        {\bf Error} & {\bf Out-Noc} & {\bf Out-All} & {\bf Avg-Noc} & {\bf Avg-All}\\ \hline
        2 pixels & 5.34 \% & 9.81 \% & 0.9 px & 1.6 px\\
        3 pixels & 3.21 \% & 6.70 \% & 0.9 px & 1.6 px\\
        4 pixels & 2.33 \% & 5.16 \% & 0.9 px & 1.6 px\\
        5 pixels & 1.86 \% & 4.25 \% & 0.9 px & 1.6 px\\
        \hline
    \end{tabular}
    \caption{IRR results on KITTI 2012}
    \label{tab:KITTI 2012_irr}
\end{table}

ScopeFlow on KITTI 2012:
\begin{table}[h!]
    \centering
    \begin{tabular}{c | c | c | c | c}
        {\bf Error} & {\bf Out-Noc} & {\bf Out-All} & {\bf Avg-Noc} & {\bf Avg-All}\\ \hline
        2 pixels & 4.36 \% & 8.30 \% & 0.7 px & 1.3 px\\
        3 pixels & 2.68 \% & 5.66 \% & 0.7 px & 1.3 px\\
        4 pixels & 1.96 \% & 4.39 \% & 0.7 px & 1.3 px\\
        5 pixels & 1.56 \% & 3.60 \% & 0.7 px & 1.3 px\\
        \hline
    \end{tabular}
    \caption{ScopeFlow results on KITTI 2012}
    \label{tab:KITTI 2012_scopeflow}
\end{table}

In addition, Fig.~\ref{fig:KITTI 2012_compare} provides a qualitative comparison to the baseline IRR model on the KITTI 2012 benchmark. As can be seen, most of the improvement provided by the ScopeFlow pipeline is in the challenging occluded and marginal pixels.

\begin{figure*}[th]
\begin{center}
    \includegraphics[width=0.95\linewidth]{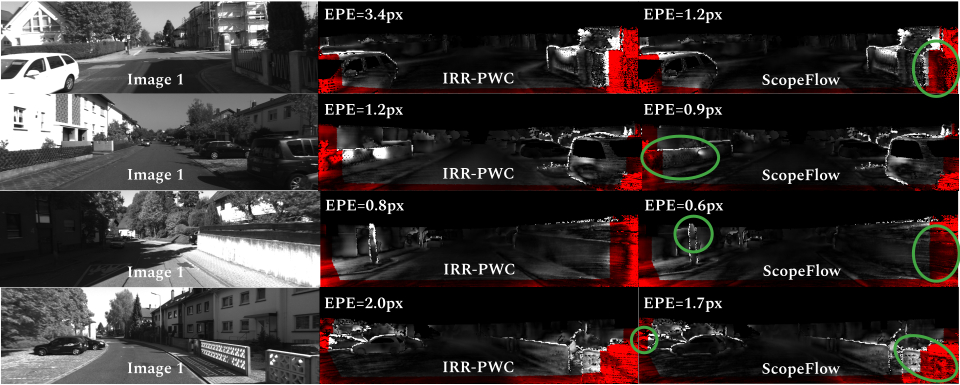}
    \end{center}
   \caption{Qualitative comparison to the IRR baseline on KITTI 2012 benchmark. Our improved training pipeline got the lowest AEPE on KITTI 2012 among all other two-frame methods, using a low parameters off-the-shelf model architecture, which has an inferior performance on the KITTI benchmarks. Occluded regions are marked in red, erroneous regions with a higher intensity. Most of the improvement provided by ScopeFlow is in these challenging marginal pixels.}
\label{fig:KITTI 2012_compare}
\end{figure*}

\subsection{KITTI 2015}
We uploaded our results to the KITTI 2015~\cite{Menze2015CVPR} benchmark. As can be seen in Tab.~\ref{tab:irr_KITTI 2015} and Tab.~\ref{tab:scopeflow_KITTI 2015}, training the IRR model with ScopeFlow pipeline improves the mean EPE by more than 12\%, in the default category of 3 pixels. Moreover, the improvement is achieved for all thresholds of outliers and for all metrics.

\medskip
IRR on KITTI 2015:
\begin{table}[h]
    \centering
    \begin{tabular}{c | c | c | c}
        {\bf Error} & {\bf Fl-bg} & {\bf Fl-fg} & {\bf Fl-all}\\ \hline
        All / All & 7.68 \% & 7.52 \% & 7.65 \%\\
        Noc / All & 4.92 \% & 4.62 \% & 4.86 \%\\
        \hline
        \end{tabular}
    \caption{IRR results on KITTI 2015}
    \label{tab:irr_KITTI 2015}
\end{table}

ScopeFlow on KITTI 2015:
\begin{table}[h]
    \centering
    \begin{tabular}{c | c | c | c}
        {\bf Error} & {\bf Fl-bg} & {\bf Fl-fg} & {\bf Fl-all}\\ \hline
        All / All & 6.72 \% & 7.36 \% & 6.82 \%\\
        Noc / All & 4.44 \% & 4.49 \% & 4.45 \%\\
        \hline
        \end{tabular}
    \caption{ScopeFlow results on KITTI 2015}
    \label{tab:scopeflow_KITTI 2015}
\end{table}

In addition, Fig.~\ref{fig:KITTI 2015_compare} provides a qualitative comparison to the leading VCN model on the KITTI 2015 benchmark, showing a clear improvement for handling non-background challenging objects. Our results are leading the category of non-background pixels, which belong to faster foreground objects.

\begin{figure*}[th]
\begin{center}
    \includegraphics[width=0.95\linewidth]{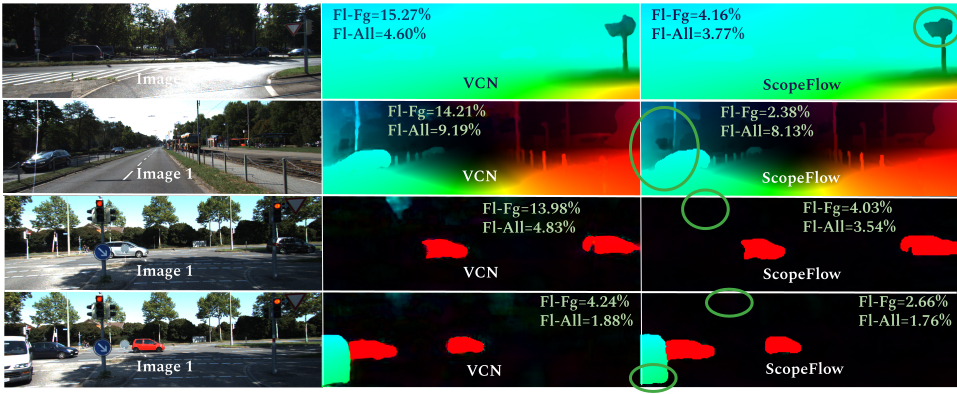}
    \end{center}
   \caption{Qualitative comparison to the VCN~\cite{yang2019vcn} method KITTI 2015 benchmark. Although the VCN architecture gets the best outlier percentage among all pixels, we have a better handling for non-background objects among all other two-frame methods.}
\label{fig:KITTI 2015_compare}
\end{figure*}

\section{Ablation visualization}
Fig.~\ref{fig:sintel_ablation} provides a demonstration of the contribution of different training changes, composing the ScopeFlow pipeline presented in our paper, to the improvement of the final flow. As expected, most of the improvements are in the marginal image area. Our method improves, in particular, the moving objects, which have many occluded and fast-moving pixels.

\begin{figure*}[th]
\begin{center}
    \includegraphics[width=0.95\linewidth]{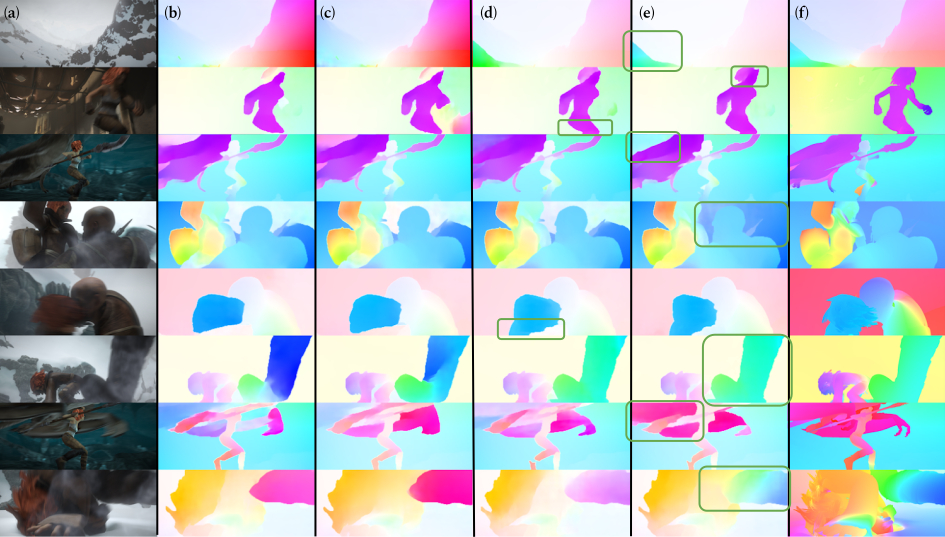}
    \end{center}
   \caption{Ablation visualization on the MPI Sintel training set. (a) First image, (b) IRR-PWC baseline, (c) ScopeFlowR (reduced regularization), (d) ScopeFlowZ (zooming schedule), (e) ScopeFlow (final model), (f) Ground Truth flow.}
\label{fig:sintel_ablation}
\end{figure*}

\section{Occlusions comparison}
In order to provide a qualitative demonstration of our improved occlusion estimation, we compared our results to the methods with the highest reported occlusion estimation. We provide a layered view of false positive, false negative and true positive predictions. All occlusion estimations created using the pre-trained models, published by the authors, and sampled from the Sintel final pass dataset. Fig.~\ref{fig:sintel_occ} shows that the model trained with the ScopeFlow pipeline improves occlusion estimation in the marginal image area and mainly for foreground objects. We used the F1 metric, with an average approach of 'micro' (the same trend presented by all averaging approaches).

\begin{figure*}[th]
\begin{center}
    \includegraphics[width=0.95\linewidth]{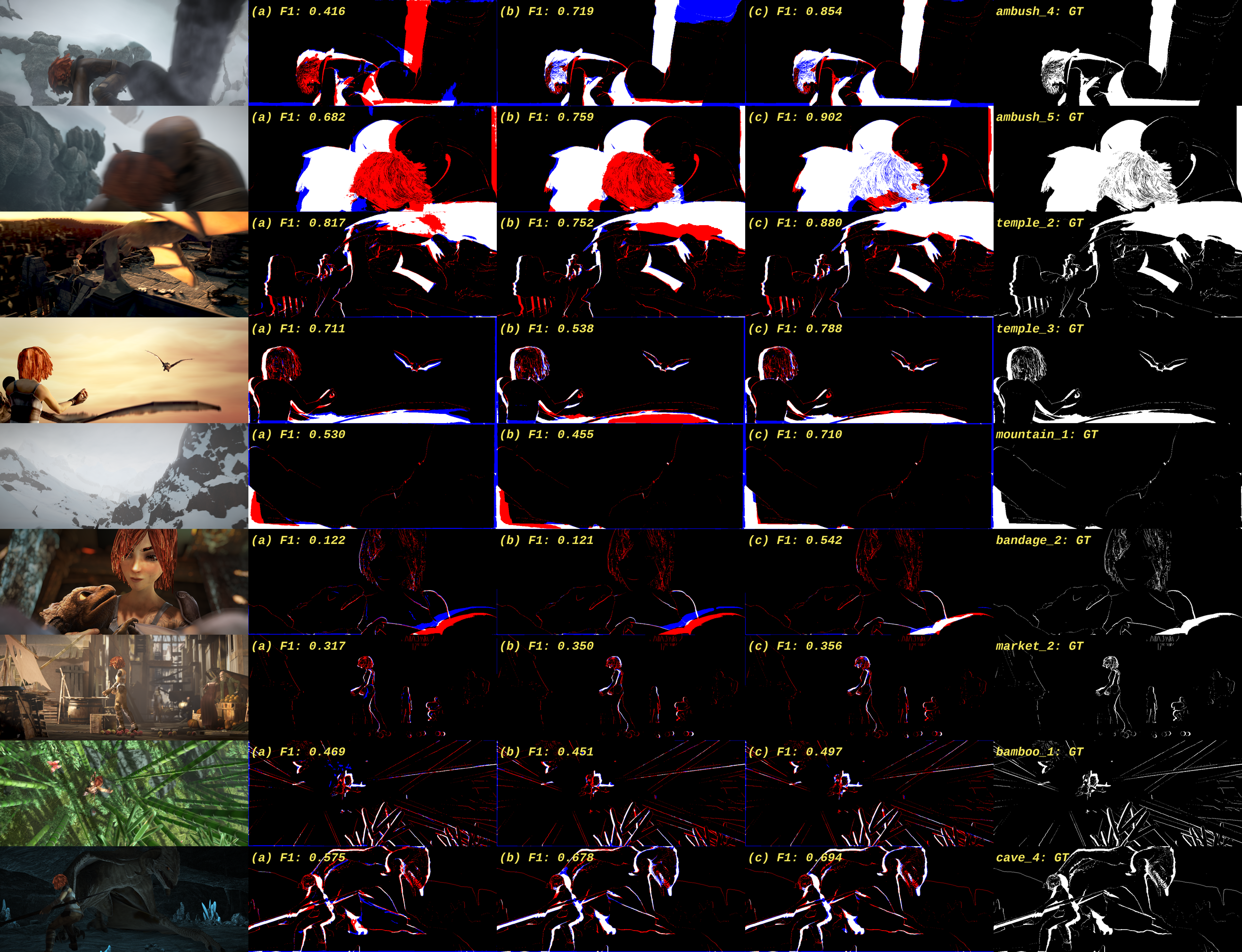}
    \end{center}
   \caption{\textbf{Occlusion comparison over Sintel final pass}. Comparison of occlusion estimations created by: (a) FlowNet-CSSR-ft-sd~\cite{ISKB18}, (b) IRR-PWC~\cite{Hur2019CVPR} baseline, and (c) ScopeFlow (ours). First frame on the left column and ground truth flow on the right column. For each occlusion map: {\color{blue}false positive} are in blue, {\color{red}false negative} in red, and true positive in white. All occlusion maps estimated using Sintel Final samples and the original models published by the authors. Our improvements are mainly for foreground objects on the image margins.}
\label{fig:sintel_occ}
\end{figure*}

\section{Public tables}
\label{sec:public_tables}

We uploaded our results to the two main optical flow benchmarks: MPI Sintel and KITTI (2012 \& 2015). In the subsections below, we provide the screenshots that capture the sizable improvements achieved by using our pipeline for training an optical flow model, with an off-the-shelf, low parameters model. Since our method can support almost any architecture, we plan, as future work, to apply it to other architectures as well.

\subsection{MPI Sintel}
We add here two screenshots of the public table: (i) the table on the day of upload (14.10.19), and (ii) the table after the official submission deadline for CVPR 2020. As shown in Fig.~\ref{fig:mpi_sintel_14_10}, our method ranks first on MPI Sintel since 14.10.19, surpassing all other methods, and leading the categories of: (i) matchable pixels, (ii) pixels far more than $10px$ from the nearest occlusion boundary, and (iii) fast-moving pixels ($>40$ pixels per frame). We also provide a screenshot taken after the official CVPR paper submission deadline, in Fig.~\ref{fig:mpi_sintel_18_11}, showing our method still leading the Sintel benchmark. We changed our method's name after the initial upload (on 14.10.19) from OFBoost to ScopeFlow.

\begin{figure*}[t]
\begin{center}
    \includegraphics[width=0.95\linewidth]{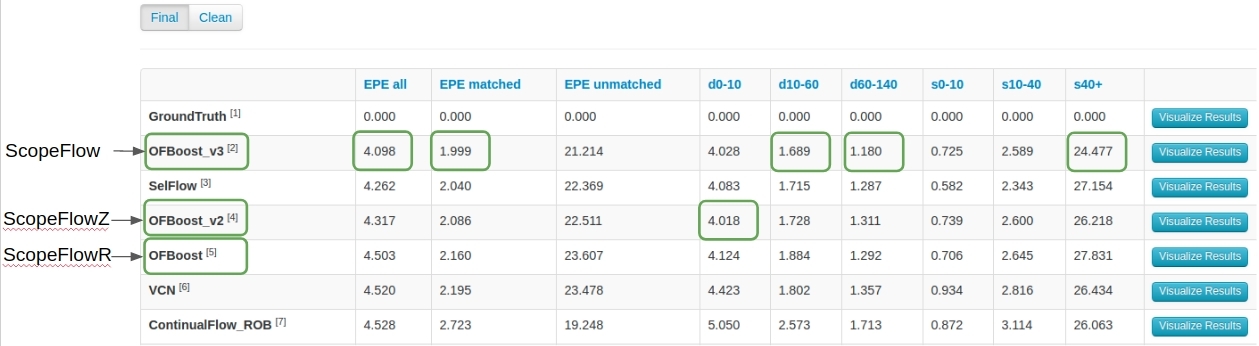}
    \end{center}
   \caption{Public Sintel table on the day of upload (taken on 14.10.19). Our method is leading the challenging final pass of the MPI Sintel benchmark. We renamed our method for clarity from OFBoost to ScopeFlow. ScopeFlowR is our method with regularization changes, ScopeFlowZ is our version with zooming changes. ScopeFlow is our final version with dynamic scoping.}
\label{fig:mpi_sintel_14_10}
\end{figure*}

\begin{figure*}[t]
\begin{center}
    \includegraphics[width=0.95\linewidth]{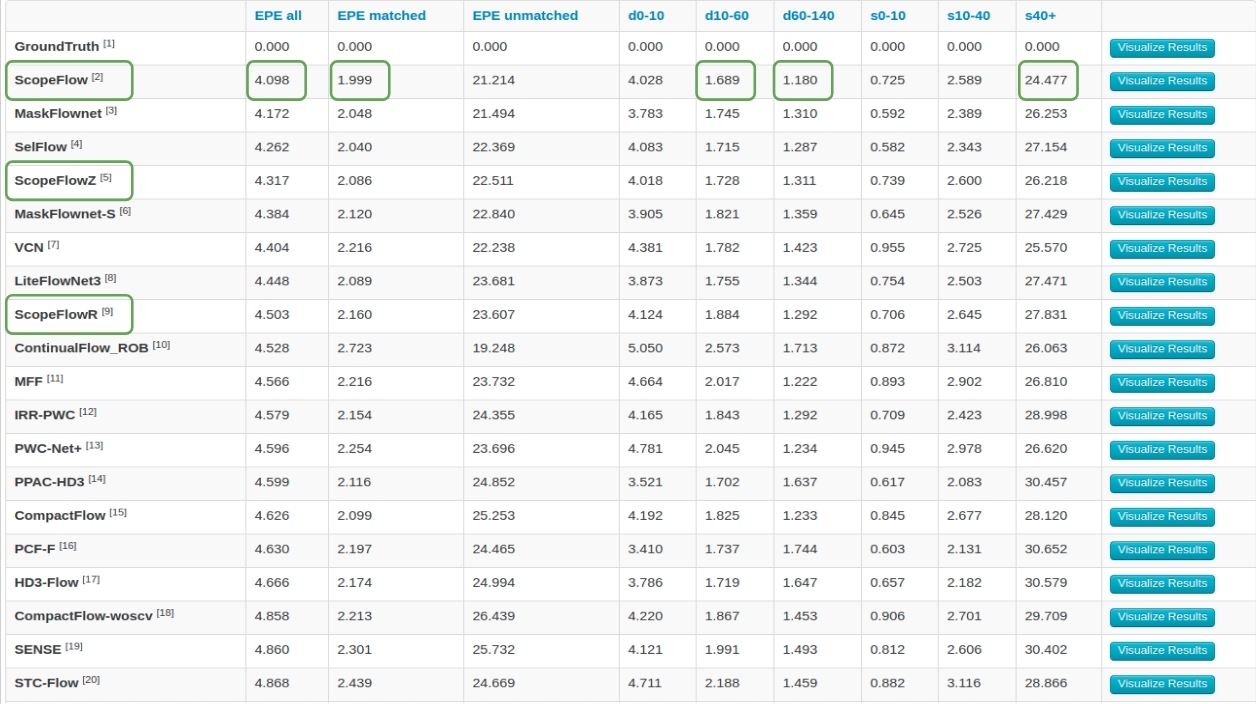}
    \end{center}
   \caption{Public Sintel table after CVPR papers submission deadline (taken on 18.11.19). Our method is still leading the main Sintel table after the addition of many new methods.}
\label{fig:mpi_sintel_18_11}
\end{figure*}

\subsection{KITTI 2012}
Fig.~\ref{fig:KITTI 2012_15_11} shows a screenshot of the KITTI 2012 flow table, with the lowest outlier threshold (of 2\%), taken on the CVPR paper submission deadline. Our method provides the lowest average EPE among all published two-frame methods, lower by 23\% from the IRR-PWC baseline results.

\begin{figure*}[t]
\begin{center}
    \includegraphics[width=0.95\linewidth]{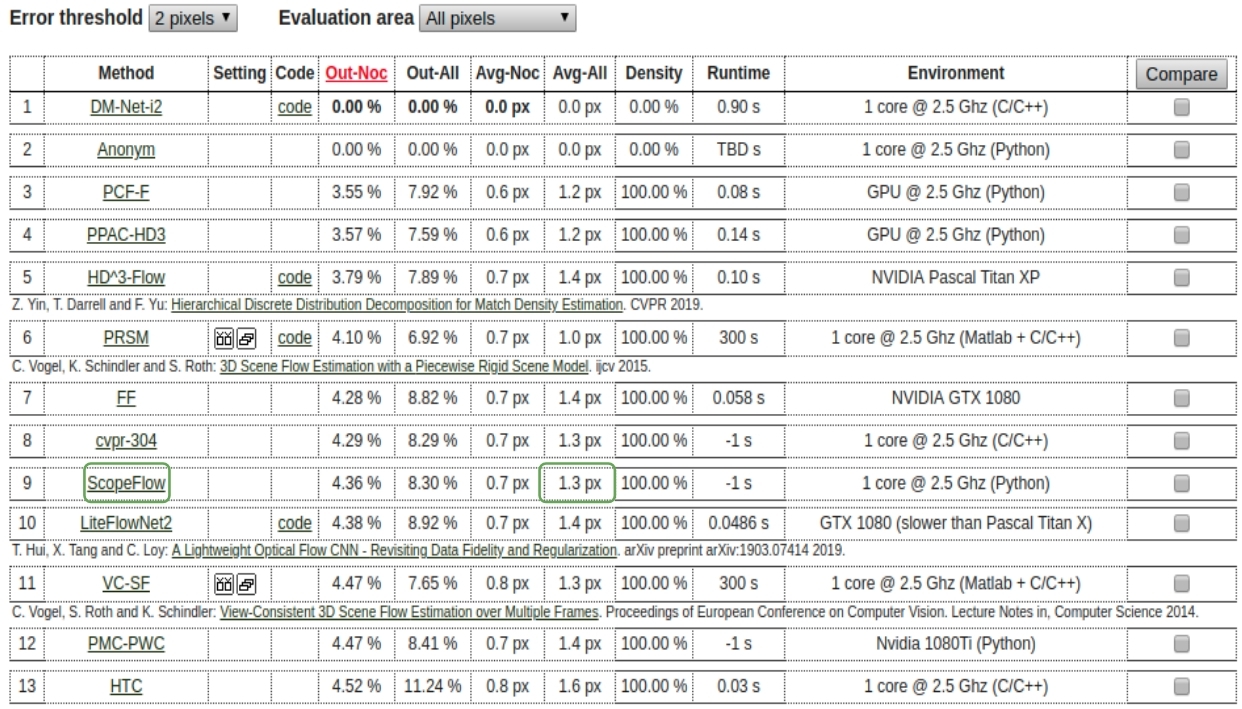}
    \end{center}
   \caption{Public KITTI 2012 flow table (with the lowest outlier threshold of 2\%) on the CVPR paper submission deadline (taken on 15.11.19). Our method is with the lowest AEPE among all published two-frame methods, lower by 23\% from the IRR-PWC baseline.}
\label{fig:KITTI 2012_15_11}
\end{figure*}

\subsection{KITTI 2015}
Fig.~\ref{fig:KITTI 2015_15_11} shows a screenshot of the KITTI 2015 flow table, taken on the CVPR paper submission deadline. Our method provides the lowest percentage of outliers, averaged over foreground regions, among all published two-frame methods. Moreover, the percentage of outliers, averaged over all ground truth pixels, is lower by more than 12\% from the IRR-PWC baseline results.

\begin{figure*}[]
    \begin{center}
    \includegraphics[width=0.95\linewidth]{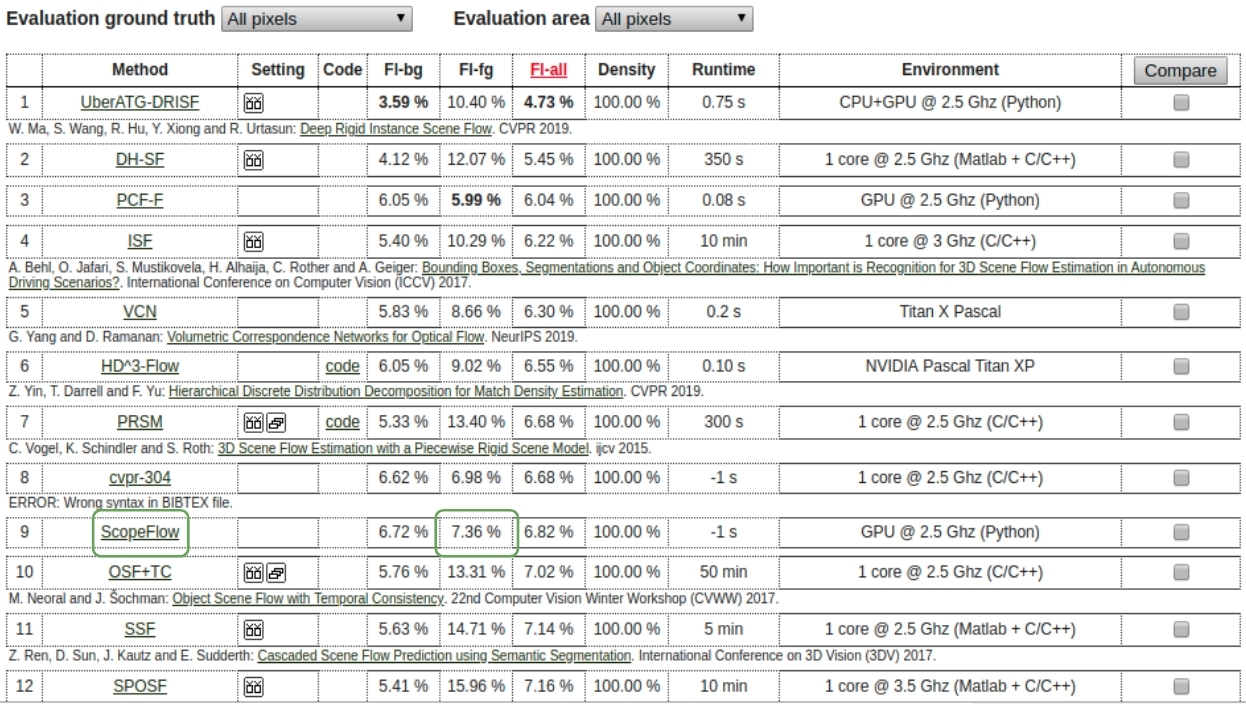}
    \end{center}
   \caption{Public KITTI 2015 flow table on the CVPR paper submission deadline (taken on 15.11.19). Our method is with the lowest percentage of foreground (objects) outliers among all published two-frame methods.}
\label{fig:KITTI 2015_15_11}
\end{figure*}

\end{document}